\documentclass{article}

\PassOptionsToPackage{numbers,compress}{natbib}

\usepackage[preprint]{neurips_2026}
\ifdefined\reviewcopy
  \usepackage{lineno}
  \AtBeginDocument{\linenumbers}
\fi

\usepackage[utf8]{inputenc}
\usepackage[T1]{fontenc}
\usepackage{hyperref}
\hypersetup{hidelinks}
\usepackage{url}
\usepackage{booktabs}
\usepackage{amsfonts}
\usepackage{nicefrac}
\usepackage{microtype}
\usepackage{xcolor}
\usepackage{amsmath,amssymb,amsthm}
\usepackage{graphicx}
\usepackage{array}
\usepackage{multirow}
\usepackage{tabularx}
\usepackage{placeins}
\usepackage{comment}

\graphicspath{{figures/}}
\DeclareGraphicsExtensions{.pdf,.png,.jpg,.jpeg}

\newcommand{\paperfigure}[2][]{%
\includegraphics[#1]{#2}%
}

\newtheorem{proposition}{Proposition}
\newtheorem{lemma}{Lemma}

\title{Engram Injection: A Study of Temporal Validity in Parametric Language-Model Memory}

\author{%
Sam Siavoshian
\And
Omar Ramadan
\And 
Amin M.\ E-A.\ Diab
\AND
Amir K. Saeed
\And
Benjamin A. Johnson
\And
Benjamin M. Rodriguez
}

\begin{document}
\raggedbottom

\maketitle

\begin{abstract}
A stored fact may stop being valid when a user changes a preference or an event replaces an earlier one. We study whether a language model can learn these changes in its parameters and answer correctly from only a question and the current time. Engram Injection stores facts in LoRA adapters and supplies the clock through FiLM conditioning on a frozen Qwen2.5-3B-Instruct model. We compare it with additive clock conditioning, prompting with the full dated history, and retrieving four dated records. We train six matched adapters across three seeds after fixing a sampler that could admit test-age examples into training. The training audit reports no such examples among 127,760 evaluated training inputs. On 62 pairs of questions asked before and after an update, at an elapsed age of seven days excluded from training, both adapter methods answer zero pairs correctly under the primary blinded assistant review with order code 1 held fixed. Full-history prompting answers 59/62 pairs correctly, and retrieval answers 62/62. Separate experiments show that adapters can store facts and that removable memory units can support local deletion. Combining fact selection with time-based version selection remains unsuccessful. These results separate learning a fact from learning when to use it. 
\end{abstract}

\section{Introduction}
\label{sec:intro}

A user first says they prefer Boston, then changes their preference to Seattle. An agent asked about the preference should return Boston for a time before the change and Seattle for a time after it. The agent must remember both the values and when each value applies.

One solution is to keep dated records outside the model and retrieve the records that apply at the query time. We use full-history prompting and retrieval as baselines for this approach. Our question is whether the model can learn the same information in its weights, so it can answer without receiving the dated records again.

We introduce \emph{Engram Injection} to study this question. Each stored fact $f_i$ has a write time $t_i$. At inference, the model receives a question $x$ and a numerical encoding $\chi(\tau_q)$ of the query time. It receives neither the fact text nor the write time as text. The target behavior is
\begin{equation}
  \mathrm{answer}(x,\tau_q)=A\!\left(f_{i(x)},\tau_q-t_{i(x)}\right).
  \label{eq:target}
\end{equation}
Here, $i(x)$ identifies the relevant memory, and $A$ returns the answer for the elapsed time $\tau_q-t_{i(x)}$. The model needs to retain information about the write time and use it with the supplied clock.

We test three parts of this problem. A shared LoRA adapter stores facts and attempts to choose the value valid at the query time. A bank of separate memory units selects facts by the meaning of the question and supports deletion by removing a unit. A joint bank attempts to select both the relevant attribute and its valid historical version. Table~\ref{tab:capabilities} shows which parts work in each configuration.

The clock channel follows ChronoState, which supplies numerical time features to a frozen model through FiLM~\cite{siavoshian2026chronostate}. FiLM lets these features scale and shift the model's hidden activations. Products of clock and write-time features can represent elapsed time. We test whether this property helps, first with a small classification task and then by comparing FiLM with additive conditioning in the language model.

In the corrected experiment, six adapters are trained across three seeds. We test 62 pairs of before-update and after-update questions at an elapsed age of seven days excluded from training. Both adapter methods score 0/62 under the primary review, compared with 59/62 for full-history prompting and 62/62 for retrieval. Earlier experiments show fact storage, changes in answer preference, and deletion in separate systems. The joint bank does not reliably combine fact selection with version selection.

\begin{table}[t]
\centering
\caption{Capabilities tested in each system. ``Demonstrated'' refers to a successful component test. ``Diagnostic only'' refers to answer preference at trained times. ``Failed'' marks an evaluated capability that did not succeed. ``--'' marks a capability not tested by that system.}
\label{tab:capabilities}
\scriptsize
\begin{tabularx}{\linewidth}{@{}Xccc@{}}
\toprule
Capability &
Shared Parametric &
Addressable Bank &
Joint Bank
\\
\midrule
Prompt-free parametric recall &
Demonstrated &
Demonstrated &
--
\\
Clock-gated supersession &
Diagnostic only &
--
&
Failed
\\
Semantic addressing &
--
&
Demonstrated &
Partial
\\
Local structural deletion &
--
&
Demonstrated &
Demonstrated
\\
Unified semantic + temporal routing &
--
&
--
&
Failed
\\
\bottomrule
\end{tabularx}
\end{table}

\paragraph{Contributions.}
\begin{itemize}
\item A method for storing dated facts in parameters and supplying the current time through a clock channel.
\item Matched comparisons of FiLM and additive conditioning with full-history and retrieval baselines, including tests that freeze or shuffle the clock.
\item Experiments on storage, editing, deletion, and joint routing that show where each component succeeds or fails.
\end{itemize}

\section{Related Work}
\label{sec:related}

\paragraph{Storing facts in parameters.}
Associative memory, fast weights, LoRA, test-time memory, and model editing provide ways to store information in parameters or other learned state~\cite{anderson1972,kohonen1972,schmidhuber1992fastweights,ba2016fastweights,hu2021lora,meng2022rome,meng2022memit,titans2025}. Cartridges learn reusable representations of context~\cite{cartridges2025}. MEGa stores events in gated adapter-like memories and selects them by semantic similarity~\cite{mega2025}. These methods help answer two questions: what should be stored, and which memory is relevant? Engram Injection also asks which stored value applies at the current time.

ROME and MEMIT modify factual associations in model parameters~\cite{meng2022rome,meng2022memit}. Our CounterFact experiment adds a dated replacement and tests whether the model can prefer the original value before the update and the replacement afterward. We do not rerun ROME or MEMIT under our protocol, so this experiment is not a direct comparison with those methods.

\paragraph{Representing time.}
Periodic and learned features can encode time, and models can also learn temporal information in their weights~\cite{kazemi2019time2vec,nylund2023timevectors}. ChronoState supplies the current time directly to a frozen model~\cite{siavoshian2026chronostate}. We use that clock channel and ask whether the model can combine it with information learned about when a fact was written.

\paragraph{Keeping records outside the model.}
Retrieval systems store dated records externally, where their dates can be inspected and filtered. Zep represents temporal relationships in external memory. LongMemEval and LoCoMo test memory over interaction histories~\cite{zep2025,longmemeval2024,locomo2024}. Our full-history and retrieval baselines use dated records. The parametric systems instead learn the records during training and receive only the question and numerical conditioning inputs at inference.

\FloatBarrier

\section{Method}
\label{sec:method}

\subsection{The frozen model and clock channel}
\label{sec:clock}

Figure~\ref{fig:architecture} shows the system with one shared memory adapter. We use a frozen Qwen2.5-3B-Instruct model~\cite{qwen2025} in BF16. The clock converts the current time $\tau$ into 31 numbers:
\begin{equation}
\chi(\tau)=
\bigl[
\sin(\tau/T_1),\cos(\tau/T_1),\ldots,
\sin(\tau/T_{15}),\cos(\tau/T_{15}),
\log(1+\tau)
\bigr],
\label{eq:chi}
\end{equation}
The 15 sinusoidal scales range from 2 seconds to 604800 seconds; the last feature is $\log(1+\tau)$. Together, they represent time at several scales. At every decoder layer except the last, linear projections convert these features into FiLM scale and shift vectors~\cite{peebles2023dit}.
\begin{equation}
\gamma_l=W_l^\gamma\chi+b_l^\gamma,\qquad
\beta_l=W_l^\beta\chi+b_l^\beta,
\end{equation}
\begin{equation}
h'_l
=
h_l+
\alpha_l
\left(
\gamma_l\odot h_l+\beta_l
\right),
\label{eq:film}
\end{equation}
The symbol $\odot$ denotes elementwise multiplication. The gate $\alpha_l$ starts at zero, so the added clock path initially leaves the backbone's activations unchanged. We train the clock projections, the gates, and rank-8 LoRA modules~\cite{hu2021lora} on the attention projections and language-model head. The backbone weights stay frozen.

\begin{figure}[tbp]
\centering
\includegraphics[width=\linewidth,trim=12 28 12 8,clip]{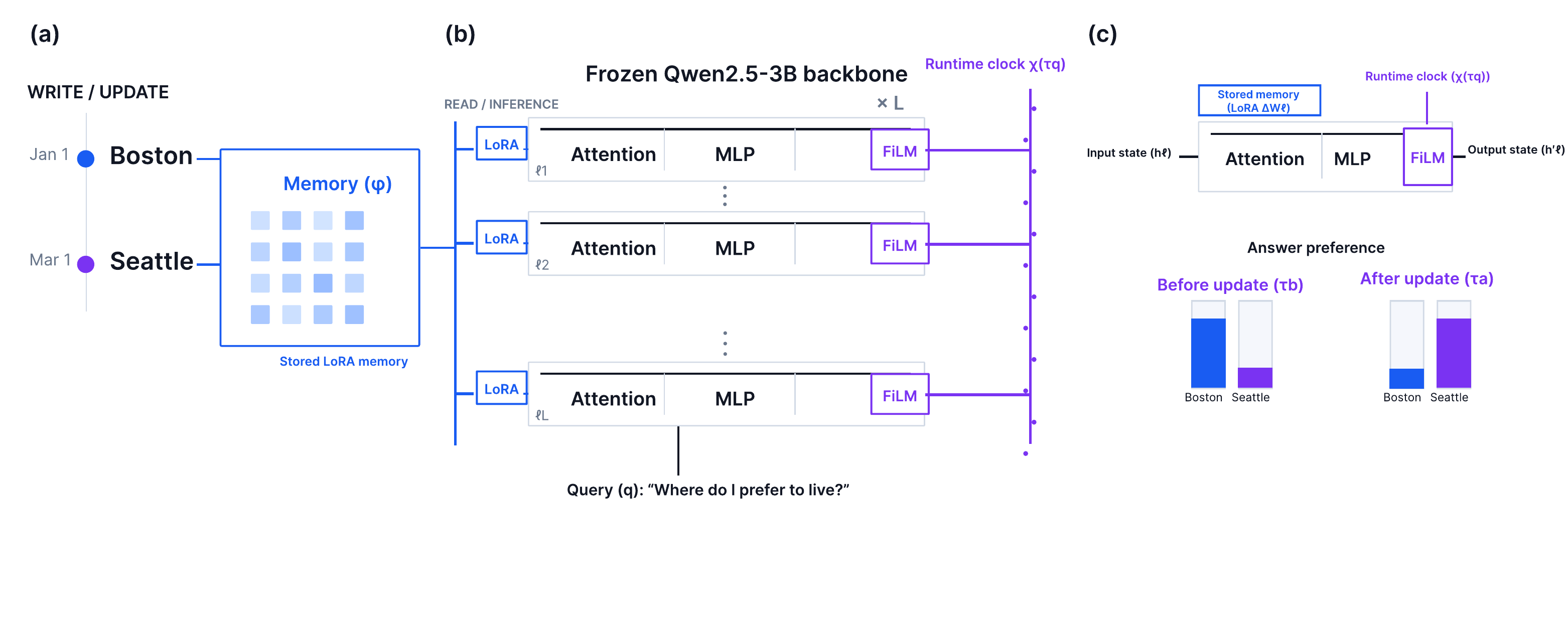}
\caption{\textbf{Engram Injection architecture.} (a) Dated updates are written into LoRA memory. (b) A frozen backbone receives the question and FiLM clock features. Clock conditioning applies to every decoder layer except the last. (c) Schematic old/new preference reversal; these bars are not measured accuracies.}
\label{fig:architecture}
\end{figure}

\subsection{Why products can represent elapsed time}
\label{sec:bilinear_main}

Selecting a valid fact often requires the elapsed time since it was written. To see how the clock features can represent this difference, consider one sinusoidal scale:
\[
\chi_k(\tau)
=
\left(
\sin(\tau/T_k),
\cos(\tau/T_k)
\right).
\]

\begin{proposition}[Relative age is bilinear]
\label{prop:bilinear}
For each scale $k$, there exists a fixed bilinear map $B_k$ such that
\[
B_k\!\left(
\chi_k(\tau_q),
\chi_k(t_i)
\right)
=
\chi_k(\tau_q-t_i).
\]
\end{proposition}

The result follows from the standard angle-difference identities:
\begin{equation}
\sin(a-b)
=
\sin a\cos b-\cos a\sin b,
\qquad
\cos(a-b)
=
\cos a\cos b+\sin a\sin b.
\label{eq:bilinear_identity}
\end{equation}

If a hidden state represents the phase of a fact's write time, FiLM can multiply it by features of the query time. The product $\gamma(\chi(\tau_q))\odot h$ can then produce features of their difference. Adding $W\chi(\tau_q)$ to $h$ does not directly produce these products in that layer.

We test this idea with a small classification task. Each example combines one of 50 fixed random fact vectors with a fact time and a query time. We sample 20,000 examples, with elapsed ages uniformly distributed from 60 seconds to ten days plus 60 seconds. The readout predicts one of five age categories, separated at one hour, one day, three days, and seven days. \emph{Age-band accuracy} is the fraction assigned to the correct category. The categories contain approximately 0.4\%, 9.6\%, 20\%, 40\%, and 30\% of examples. Always predicting the largest category scores about 40\%; uniform guessing scores 20\%.

A linear readout scores 0.928 with FiLM and 0.396 with additive conditioning. With one hidden layer, the scores rise to 0.987 and 0.982. FiLM therefore helps the linear readout, while a nonlinear additive readout recovers almost the same information. This task measures age classification, not language-model answer accuracy. Section~\ref{sec:comparison} compares both paths in the full model. The joint router also receives clock--boundary products, so its failure cannot be explained simply by the absence of multiplication.

A separate ridge-regression probe predicts relative phase with $R^2=0.003$ from raw phase features and $R^2=1.000000$ from four explicit bilinear products. The products make relative phase accessible to a linear readout. The model still has to learn useful write-time information in its hidden states. Appendix~\ref{app:bilinear} gives the proof and explains the limits imposed by periodic features.

\subsection{Shared Parametric Memory}
\label{sec:e1}

Shared Parametric Memory stores all facts in one LoRA adapter. For each training example, we sample a fact $i$, a query time $\tau_q$, and a question $x$ without a timestamp in its text. The target answer $y$ is the value valid at that time. We minimize negative log-likelihood on the answer tokens:
\begin{equation}
\min_{\phi}
\;
\mathbb{E}_{i,\tau_q,(x,y)}
\left[
-\log
p_{\theta,\phi}
\bigl(
y\mid x;\chi(\tau_q)
\bigr)
\right],
\label{eq:e1obj}
\end{equation}
Here, $\theta$ denotes the frozen backbone weights and $\phi$ the trainable memory weights.

The write time $t_i$ determines the training answer but is not supplied as a separate input vector or included in the inference prompt. Any write-time information used at inference must therefore come from the learned parameters. The earlier rank-8 configuration has approximately 9.58M trainable parameters. The adapters in the matched comparison contain 40.91 MiB of trainable tensors; this is a different measurement from the earlier serialized sizes.

\subsection{Selecting attributes and historical versions}
\label{sec:joint_router}

The addressable bank stores facts in separate, removable units. Each unit has a semantic key $c_a$ derived from a representation of its question. A query representation $e(q)$ is compared with these keys by cosine similarity. A softmax turns the similarities into selection weights:
\begin{equation}
s_a(q)
=
\operatorname{cos}
\left(
e(q),c_a
\right),
\qquad
p_{\rm sem}(a\mid q)
=
\operatorname{softmax}_a
\left(
s_a/T_s
\right).
\label{eq:semantic_gate}
\end{equation}
Deleting a memory removes that unit's contribution without changing the parameters of neighboring units.

The joint bank gives each attribute a separate unit for each historical version, indexed by $(a,v)$. A temporal head predicts $p_{\rm time}(v\mid a,\tau)$ from the clock, the version's validity boundaries, and their multiplicative interactions. We multiply the attribute-selection weight by the version-selection weight:
\begin{equation}
g_{a,v}(q,\tau)
=
p_{\rm sem}(a\mid q)
\,
p_{\rm time}(v\mid a,\tau).
\label{eq:joint_gate}
\end{equation}
The resulting weight selects both what the question is about and which version applies at the query time.

\subsection{Teaching the model to use updated facts}
\label{sec:curriculum}

We train time-based selection using changed facts. If a preference changes from blue to green, the same question should favor blue before the update and green afterward. We compare the two answers through their log-probability difference:
\begin{equation}
m(\tau)
=
\log p(\mathrm{new}\mid q,\chi(\tau))
-
\log p(\mathrm{old}\mid q,\chi(\tau)).
\label{eq:margin}
\end{equation}
A positive $m(\tau)$ favors the new value; a negative value favors the old one. We report the fraction of queries with the expected sign on each side of the update:
\begin{align}
\mathrm{new\_after}
&=
\Pr
\left[
m(\tau_{\rm after})>0
\right],
\\
\mathrm{old\_before}
&=
\Pr
\left[
m(\tau_{\rm before})<0
\right].
\label{eq:behavioral}
\end{align}

Some configurations also receive an ordinal code $z$, a numerical tag associated with write order during training. Their answer distribution is $p(y\mid q;\chi(\tau),z)$. We omit $z$ from the equations to keep them readable. To test the clock alone, we must hold this code fixed. We compare ordinary answer-token NLL training, an order tag, a temporal-margin loss, and their combination. With $[z]_+=\max(0,z)$, the margin loss is

\begin{equation}
    \mathcal{L}_{\rm temp}
    =
    [\mu+m(\tau_{\rm before})]_+
    +
    [\mu-m(\tau_{\rm after})]_+,
\end{equation}
\begin{equation}
    \mathcal{L}
    =
    \mathcal{L}_{\rm NLL}
    +
    \lambda
    \mathcal{L}_{\rm temp}.
    \label{eq:temporal_margin}
\end{equation}

\begin{lemma}[Margin loss and preference reversal]
    \label{lem:margin_zero}
    Let
    \[
    m(\tau)
    =
    \log p(\mathrm{new}\mid q,\chi(\tau))
    -
    \log p(\mathrm{old}\mid q,\chi(\tau)).
    \]
    If \(\mathcal{L}_{\rm temp}=0\) for margin threshold \(\mu>0\), then
    \[
    m(\tau_{\rm before})\le -\mu < 0
    \qquad\text{and}\qquad
    m(\tau_{\rm after})\ge \mu > 0.
    \]
    Consequently the model strictly prefers the old value before the update boundary and the new value afterward.
\end{lemma}

The loss asks the model to favor each answer by at least $\mu$. The preference metrics only ask whether the correct answer has higher probability. A model can therefore pass those metrics without reaching zero margin loss. Appendix~\ref{app:margin} gives the proof.

We freeze the clock at one value or shuffle its values across queries while keeping the ordinal code fixed. This tests what changes when the clock changes. Earlier tests changed both the clock and the code; we report those separately.

\subsection{Evaluation}
\label{sec:data}

PersonaFacts-100 contains 100 dated synthetic facts about identity, preferences, relations, and events over 180 days. The earlier main update experiment uses 25 old/new pairs in a 140-fact configuration. A separate experiment uses 100 update pairs. We compute the correct answer from the fact database at each query time. Parametric systems receive no timestamp text. External-memory systems receive dated records and an explicit query day, allowing them to select records valid on that day.

CounterFact~\cite{meng2022rome,meng2022memit} tests whether a dated edit can replace a fact already known by the base model. A forget/retain split from TOFU~\cite{maini2024tofu} tests deletion. We do not rerun ROME or MEMIT with the same backbone and protocol, so their published scores are not matched baselines here.

The joint-bank benchmark asks the system to select an attribute and its historical version. Four attributes have multiple versions. The test includes new question paraphrases without timestamps, query times inside excluded intervals, and queries near update boundaries. We measure the accuracy of the selected attribute, selected version, and their combination; exact answer matches; correctness when the clock changes; and changes to retained memories after deletion.

We seed Python, NumPy, and PyTorch in each multi-seed experiment and save the model, optimizer, and random-number-generator states together. Appendix~\ref{app:details} describes the sampling, controls, and additional evaluations.

\section{Results}
\label{sec:results}

\subsection{Earlier component experiments}
\label{sec:components}

\subsubsection{Fact storage and answer preference}
\label{sec:legacy_clock_results}

The earlier experiments test whether the system can recall facts and prefer the appropriate answer at times used during training. A preference score compares answer likelihoods; it does not require the model to generate a correct answer. Section~\ref{sec:trained_code} reports that stricter test. Figure~\ref{fig:original_temporal} summarizes the earlier preference experiments.

The shared adapter recalls 92 of 100 PersonaFacts-100 items without receiving the fact text in the query prompt. It can therefore store and recover facts in this controlled setting.

Supplying the clock alone does not reliably make the model prefer the new value after an update. On 25 update pairs, it gives a mean $\mathrm{new\_after}$ of $0.38$ across four seeds. An order tag without the margin loss gives $0.30$. Margin-only training gives $0.82$ across two seeds. Combining the order tag and margin loss gives $0.79$ across ten seeds, ranging from $0.72$ to $0.92$. The combined system also gives $\mathrm{old\_before}=0.88$ and recall $0.87$.

On 100 update pairs, the combined system gives $\mathrm{new\_after}=0.948$, $\mathrm{old\_before}=0.99$, and recall $0.945$ across five seeds. The clock-only baseline gives $0.675$, $0.38$, and $0.63$ across two seeds. The model still changes its answer preference at this larger tested size.

\begin{table}[htbp]
\centering
\caption{
Temporal supersession across training arms and update-set sizes.
The range is shown for the earlier ten-seed configuration.
}
\label{tab:temporal_all}
\small
\begin{tabular}{@{}llcccc@{}}
\toprule
Pairs &
Training arm &
Seeds &
$\mathrm{new\_after}$ &
$\mathrm{old\_before}$ &
Recall
\\
\midrule
25 &
Time channel only &
4 &
0.38 &
0.65 &
0.77
\\
25 &
Order tag only &
2 &
0.30 &
0.76 &
0.77
\\
25 &
Margin only &
2 &
\textbf{0.82} &
0.72 &
0.81
\\
25 &
Order + margin &
\textbf{10} &
0.79 [0.72,0.92] &
\textbf{0.88} &
\textbf{0.87}
\\
\midrule
100 &
Time channel only &
2 &
0.675 &
0.38 &
0.63
\\
100 &
Order + margin &
5 &
\textbf{0.948} &
\textbf{0.99} &
\textbf{0.945}
\\
\bottomrule
\end{tabular}
\end{table}

The margin loss produces the largest improvement in these tests. The margin-only experiment uses just two seeds, so the extra benefit of the order tag is uncertain.

Shuffling the timestamps lowers $\mathrm{new\_after}$ from $0.79$ to $0.28$; freezing the clock lowers it to $0.44$. This earlier procedure can also change the ordinal code. The later controls hold the code fixed to separate its effect from the clock's effect.

\begin{figure}[tbp]
\centering
\includegraphics[width=\linewidth,trim=12 28 12 8,clip]{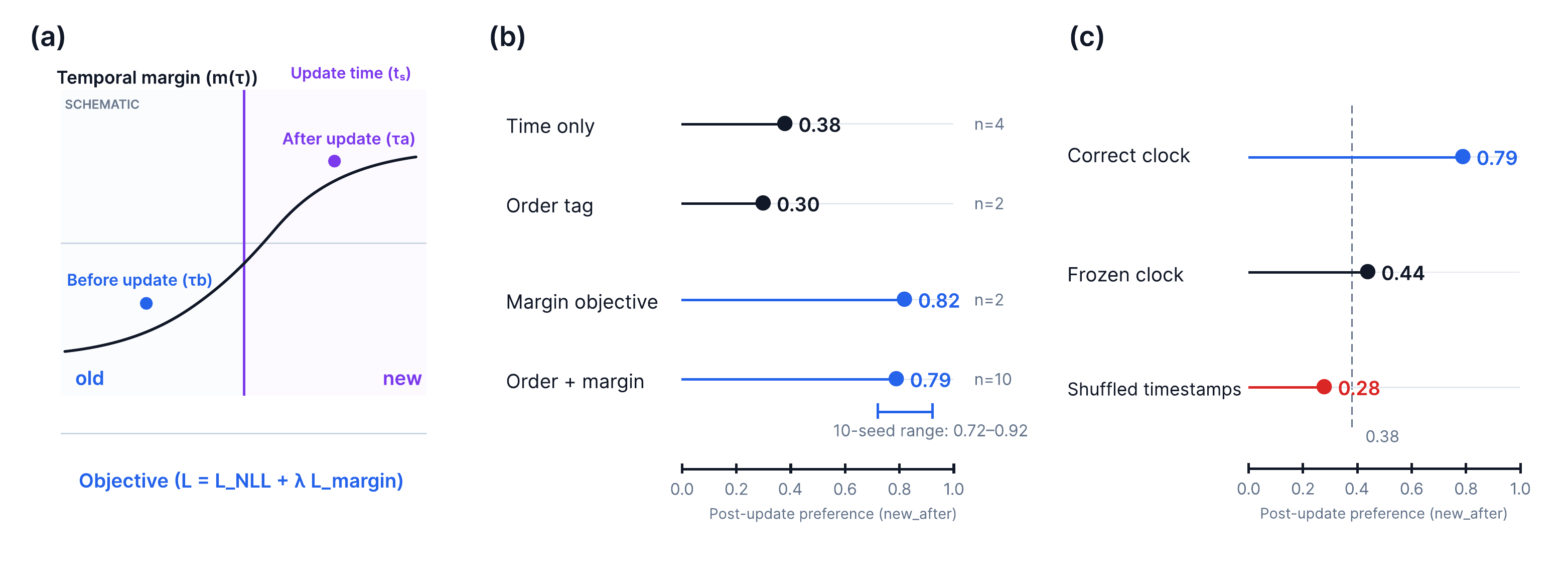}
\caption{\textbf{Earlier answer-preference tests.} (a) The target margin favors the old value before an update and the new value afterward. (b) Training-arm scores measure post-update likelihood preference; $n$ is the number of seeds. (c) Frozen and shuffled controls lower this score. These earlier controls can also change the order code. Section~\ref{sec:trained_code} tests generated answers with the code fixed.}
\label{fig:original_temporal}
\end{figure}

\subsubsection{Editing and deletion}
\label{sec:external}

Table~\ref{tab:external} reports the editing and deletion experiments. In CounterFact, a dated update replaces information already known by the base model. Edit efficacy ranges from $0.96$ to $1.00$, falling to $0.00$--$0.12$ when the clock is frozen. Before the update, the old-value preference remains $0.99$. The model can thus retain a preference for the historical answer as well as the replacement. Fair-locality, which measures preservation of held-out facts, is $0.80$, compared with $0.87$ for the base model.

The TOFU split tests deletion in the separate bank of removable memories. Recall on the forget and retain sets is $0.01$ and $0.00$ before injection, rising to $0.98$ and $0.97$ afterward. Removing the forget-set units returns their recall to $0.01$, while retain-set recall stays at $0.97$. Deletion preserves the tested retained memories.

\begin{table}[htbp]
\centering
\caption{
External editing and deletion evaluation.
CounterFact evaluates the clock-conditioned editing configuration; TOFU evaluates structural deletion in the separate addressable bank.
}
\label{tab:external}
\small
\begin{tabular}{@{}llccc@{}}
\toprule
Benchmark &
Metric &
Base &
Injected &
Intervention
\\
\midrule
CounterFact &
Efficacy &
0.14 &
\textbf{0.96--1.00} &
0.00--0.12
\\
CounterFact &
Fair-locality &
0.87 &
0.80 &
0.81
\\
CounterFact &
Pre-update old preference &
0.86 &
\textbf{0.99} &
0.99
\\
\midrule
TOFU &
Forget-set recall &
0.01 &
0.98 &
\textbf{0.01}
\\
TOFU &
Retain-set recall &
0.00 &
0.97 &
\textbf{0.97}
\\
\bottomrule
\end{tabular}
\end{table}

\subsubsection{Answer preference across model sizes}
\label{sec:model_scale_main}

We also test selected configurations with 3B, 7B, and 14B Qwen2.5-Instruct backbones. These are separate runs from the ten-seed experiment in Table~\ref{tab:temporal_all}. Post-update preference scores are $0.813$, $0.760$, and $0.747$, and recall is $0.89$, $0.85$, and $0.86$, respectively. CounterFact efficacy is $1.00$ at all three sizes.

\begin{table}[htbp]
\centering
\caption{Selected configurations at three Qwen2.5-Instruct sizes. These runs are separate from the ten-seed experiment in Table~\ref{tab:temporal_all}.}
\label{tab:model_scale_main}
\small
\begin{tabular}{@{}lccc@{}}
\toprule
Size &
Controlled flip &
Recall &
CounterFact efficacy
\\
\midrule
3B &
0.813 &
0.89 &
1.00
\\
7B &
0.760 &
0.85 &
1.00
\\
14B &
0.747 &
0.86 &
1.00
\\
\bottomrule
\end{tabular}
\end{table}

All three model sizes show preference reversal, but larger models do not consistently improve the score (Table~\ref{tab:model_scale_main}).

\FloatBarrier
\subsection{Full-history, retrieval, and additive comparisons}
\label{sec:comparison}

We compare six systems using seeds 0, 1, and 2, with the same 140 facts and 25 update pairs per seed. P0 puts the full dated history in the prompt. P1 uses MiniLM embeddings to retrieve four records and then sorts them by date. Their as-of variants, P0-asof and P1-asof, first discard records written after the query time. These baselines receive the query day as text. The FiLM and additive systems receive a question, numerical clock features, and a fixed order code; the fact records are omitted.

Each additive run matches a FiLM run in data, optimizer, LoRA rank, order code, margin loss, and number of training steps. The additive path uses $h'=h+\alpha\beta(\chi(\tau))$, removing the multiplication by $h$ from Equation~\ref{eq:film}. Its projection bias starts at 0.01. If both the projection and the gate started at zero, neither could receive the gradients needed to open the path.

\paragraph{Scoring and conditioning.}
Literal matching checks whether the expected value appears in the generated answer. Each update produces two questions: one before the update and one after it. A pair counts as correct only when both answers match. This is stricter than testing which answer has higher likelihood (Equation~\ref{eq:behavioral}). The earlier comparison uses code 0, which is not trained by the current-value curriculum. Tables~\ref{tab:comparison} and~\ref{tab:heldout} retain those results. The corrected test uses trained code 1 as its primary setting and code 2 as a second check. Each code stays fixed across query times and clock controls.

\paragraph{Training endpoints.}
Table~\ref{tab:comparison} evaluates the same before/after times used in margin training. These scores describe behavior at trained times. Recall uses 240 questions across the three seeds; each temporal column uses 75 update pairs. We pool counts before converting them to percentages. FiLM scores higher than additive conditioning at these times, but both score below the as-of baselines.
\begin{table}[t]
\centering\small
\caption{Three-seed literal-match comparison at training endpoints. Before, after, and both use 75 pairs; recall uses 240 questions. Mean prompt tokens are descriptive query costs.}
\label{tab:comparison}
\begin{tabular}{@{}lrrrrr@{}}\toprule
System & Recall (\%) & Before (\%) & After (\%) & Both (\%) & Tokens\\\midrule
P0 & 61.67 & 30.67 & 86.67 & 25.33 & 3087.3\\
P0-asof & 62.08 & 89.33 & 92.00 & 89.33 & 1717.9\\
P1 & 61.25 & 10.67 & 89.33 & 9.33 & 148.0\\
P1-asof & 62.50 & 88.00 & 92.00 & 88.00 & 145.8\\
Additive & 97.92 & 77.33 & 36.00 & 25.33 & 6.0\\
FiLM & 97.50 & 81.33 & 41.33 & 33.33 & 6.0\\
\bottomrule\end{tabular}
\end{table}

\paragraph{Earlier code-0 temporal-offset tests.}
The earlier test was intended to hold out ages from 5 to 9 days, where age is the query time minus the relevant fact's write time. A pair qualifies only if both values remain valid through that age band and neither known margin-training time falls inside it. This leaves 55 pairs across three seeds. It tests new ages of learned facts. However, a sampler fallback could draw a time inside the excluded band, so training/test separation is unverified. The corrected test below is our primary evaluation. Because the seeds have different eligible counts, pooled accuracy can differ from the average seed accuracy.

Full-history prompting with future records removed answers 47/55 pairs correctly, and retrieval answers 48/55. FiLM answers 7/55 and additive conditioning 6/55. Both as-of baselines perform better under this earlier fixed-zero-code procedure.

\begin{table}[t]
\centering\small
\caption{Earlier code-0 temporal-offset tests with unverified training/test separation. Literal scores are pooled counts out of 55 eligible pairs. Judge scores are supplementary, unvalidated semantic assessments of the same generated answers.}
\label{tab:heldout}
\begin{tabular}{@{}lrrrr@{}}\toprule
System & Before & After & Both (\%) & Judge both\\\midrule
P0 & 33/55 & 48/55 & 54.5 & 35/55\\
P0-asof & 49/55 & 49/55 & 85.5 & 53/55\\
P1 & 11/55 & 49/55 & 18.2 & 10/55\\
P1-asof & 48/55 & 51/55 & 87.3 & 55/55\\
Additive & 35/55 & 10/55 & 10.9 & 1/55\\
FiLM & 43/55 & 14/55 & 12.7 & 2/55\\
\bottomrule\end{tabular}
\end{table}

\paragraph{Scoring uncertainty.}
Literal matching may miss a valid paraphrase or accept an answer that includes the correct value and then contradicts it. We also score these earlier answers with a local Qwen2.5-7B-Instruct judge that receives no system label. It gives 53/55 for P0-asof, 55/55 for P1-asof, 1/55 for additive, and 2/55 for FiLM. These judgments have not been checked by human reviewers. Both scoring methods favor the as-of baselines, although they disagree substantially on the adapter answers.

\subsection{Corrected tests at an excluded query age}
\label{sec:trained_code}

We fix the sampler and train both adapter methods again from scratch for 1200 optimizer steps per seed. FiLM and additive runs use the same facts, training schedule, and budget for seeds 0, 1, and 2. The corrected sampler excludes the entire closed interval from 5 to 9 days after a referenced fact was written. If no permitted time is available, it skips the draw and records why.

The completion audit reports that the sampler was checked before freezing the test and starting training. It then checks all 127,760 inputs actually evaluated for NLL and margin training against both the sampling policy and an independent time calculation. It reports zero inputs in the excluded age band and zero exact overlaps with frozen test timestamps. These checks took place before evaluation. Appendix~\ref{app:clean_rerun} describes which audit checks we could verify from the supplied reports.

The test contains 18, 20, and 24 eligible pairs for the three seeds, totaling 62. In each question, the relevant fact is seven days old. The model learns the fact values during training; seven-day query ages are excluded. Code 1 is the primary setting and code 2 checks sensitivity to the code choice. We keep each code fixed before and after the update and under live, frozen, and shuffled clocks. During training, these codes remain associated with write order.

We rerun P0-asof and P1-asof on the same test questions and correct answers. Both remove future writes and receive the query day as text. P0 includes every remaining dated write. P1 uses all-MiniLM-L6-v2 to retrieve the four most similar remaining writes by cosine similarity, then sorts them by date. All systems decode greedily, generate at most 24 tokens, and stop at EOS or a newline. Baselines use a chat template; adapters use the question-plus-clock/code interface used in their training.

\paragraph{Complete-answer review.}
We use blinded assistant review to judge the whole answer rather than just search for a matching string. It labels 891 unique question/gold-answer/prediction triples, covering 1736 answer occurrences. A pair succeeds only if both answers are correct. The strict score treats ambiguous answers as incorrect; an upper-bound score accepts them. The unique triples receive 167 correct, 680 incorrect, and 44 ambiguous labels. The run reports that mappings and decision hashes pass its checks. This is assistant review, with no automated judge or independent human validation.

\begin{table}[t]
\centering\small
\caption{Corrected held-out 7-day temporal-offset tests on stored facts. All entries are pooled counts out of 62 pairs. Strict blinded assistant review rejects ambiguity; upper bounds accept it. Both baselines are rerun on the same queries.}
\label{tab:trained_code}
\begin{tabular}{@{}lllrrr@{}}\toprule
System & Code & Clock & Literal & Reviewed & Upper bound\\\midrule
P0-asof & -- & asof & 52/62 & 59/62 & 59/62\\
P1-asof & -- & asof & 54/62 & 62/62 & 62/62\\
Additive & 1 & frozen & 12/62 & 0/62 & 1/62\\
Additive & 1 & live & 11/62 & 0/62 & 1/62\\
Additive & 1 & shuffled & 11/62 & 0/62 & 2/62\\
Additive & 2 & frozen & 6/62 & 0/62 & 3/62\\
Additive & 2 & live & 9/62 & 0/62 & 2/62\\
Additive & 2 & shuffled & 3/62 & 0/62 & 3/62\\
FiLM & 1 & frozen & 9/62 & 0/62 & 0/62\\
FiLM & 1 & live & 11/62 & 0/62 & 0/62\\
FiLM & 1 & shuffled & 9/62 & 0/62 & 0/62\\
FiLM & 2 & frozen & 6/62 & 0/62 & 0/62\\
FiLM & 2 & live & 10/62 & 1/62 & 1/62\\
FiLM & 2 & shuffled & 9/62 & 0/62 & 0/62\\
\bottomrule\end{tabular}
\end{table}

Both adapter methods answer 0/62 pairs correctly with code 1, under every clock setting (Table~\ref{tab:trained_code}; Figure~\ref{fig:temporal_results}). P0-asof answers 59/62 correctly (95.16\%), and P1-asof answers 62/62 (100\%). Descriptive Wilson 95\% intervals are 0--5.83\% for each adapter, 86.71--98.34\% for P0-asof, and 94.17--100\% for P1-asof. Table~\ref{tab:clean_seeds} lists the primary counts for each seed.

\begin{table}[tbp]
\centering\small
\caption{Primary code-1 live-clock reviewed counts by seed. Both temporal answers must be correct.}
\label{tab:clean_seeds}
\begin{tabular}{@{}lrrrr@{}}\toprule
System & Seed 0 & Seed 1 & Seed 2 & Pooled\\\midrule
FiLM & 0/18 & 0/20 & 0/24 & 0/62\\
Additive & 0/18 & 0/20 & 0/24 & 0/62\\
P0-asof & 18/18 & 20/20 & 21/24 & 59/62\\
P1-asof & 18/18 & 20/20 & 24/24 & 62/62\\
\bottomrule\end{tabular}
\end{table}

No primary pair changes from failure to success when the live clock is replaced by a frozen or shuffled clock. Each exact paired McNemar test gives $p=1$. Because the adapters fail every pair in all three settings, these scores cannot show whether their internal computations use the clock.

With code 2, FiLM answers one pair correctly under the live clock and none under either control. Additive conditioning answers none in all settings. FiLM's live/control difference is 1.61 percentage points, with exact two-sided $p=1$ and a within-seed pair-bootstrap interval of 0--4.84 points. One successful pair is insufficient to establish a clock benefit. Accepting all ambiguous answers still gives at most 3/62 for any adapter setting.

\paragraph{Literal-score disagreement.}
Under code 1 and the live clock, literal matching gives 11/62 for both adapters, 52/62 for P0-asof, and 54/62 for P1-asof. Whole-answer review can reject a response that contains the correct value but denies knowing it, lists conflicting versions, or refers to another entity. It can also accept clear paraphrases that literal matching misses. The run reports 644 disagreements across answer occurrences and records the complete answers and review reasons.

\begin{figure}[tbp]
\centering\paperfigure[width=\linewidth]{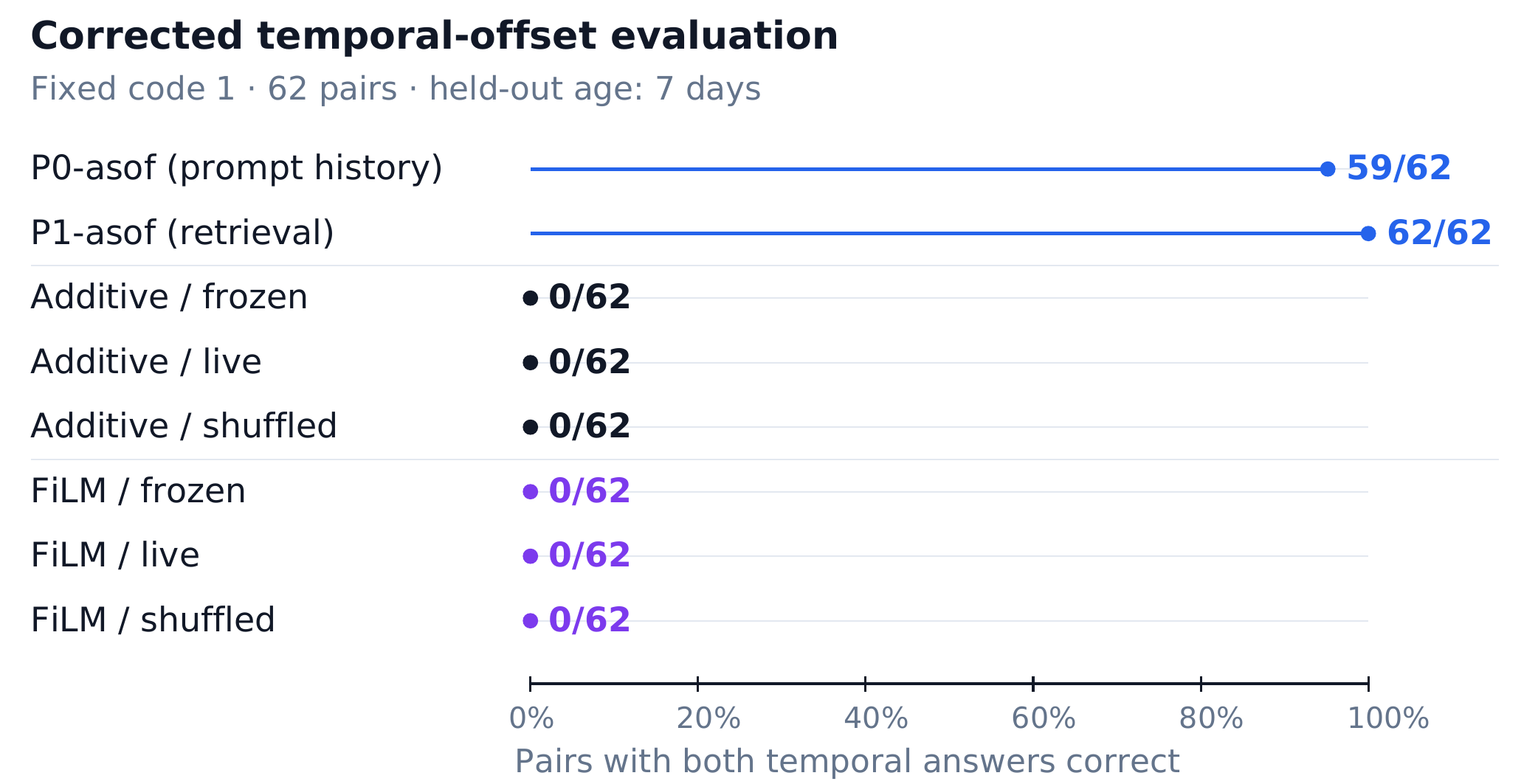}
\caption{Primary reviewed scores on 62 before/after pairs at a seven-day age excluded from training. Code 1 stays fixed across query times and clock controls. Both as-of baselines are rerun on these queries. With all adapter scores at zero, the controls cannot determine internal clock use.}
\label{fig:temporal_results}
\end{figure}

Fixing the sampler and using trained codes does not make the adapters answer these time-dependent questions reliably. Section~\ref{sec:limitations} describes the test's scope and remaining uncertainties.

\subsection{Earlier endpoint clock interventions}
\label{sec:clock_results}

In the earlier tests, changing the clock while keeping code 0 fixed changes the mean likelihood margin by 1.014 nats for FiLM and 2.858 nats for additive conditioning. This shows that the output responds to the clock, but does not show that the answer is valid at the query time. Table~\ref{tab:fixed_controls} compares generated answers at training times with code 1 fixed, under live and frozen clocks. These are separate from the excluded-age test above. Other earlier tests changed the clock and code together; their scores cannot isolate the clock's effect. Section~\ref{sec:components} reports those component experiments separately.

\begin{table}[t]
\centering\small
\caption{Literal pairs correct at both training endpoints with ordinal code one fixed. Each seed contains 25 pairs. Frozen controls hold the runtime clock constant.}
\label{tab:fixed_controls}
\begin{tabular}{@{}lrrrr@{}}\toprule
 & \multicolumn{2}{c}{FiLM (\%)} & \multicolumn{2}{c}{Additive (\%)}\\
Seed & Live & Frozen & Live & Frozen\\\midrule
0 & 64 & 24 & 32 & 4\\
1 & 56 & 12 & 56 & 16\\
2 & 0 & 0 & 52 & 8\\\bottomrule
\end{tabular}
\end{table}

\subsection{Joint attribute and version selection}
\label{sec:joint_result}

The earlier joint-bank configuration selects the correct attribute with accuracy 0.94, the correct version with accuracy 0.42, and their combination with accuracy 0.35. Exact answer accuracy is also 0.35. Correctness when the clock shifts is zero. These results come from a different configuration than the replication below.

We repeat the joint-bank experiment across three seeds, with 96 queries per seed. It is separate from the shared-adapter comparison. Mean attribute-selection accuracy is 0.7292, version-selection accuracy 0.4167, combined accuracy 0.2188, and exact answer accuracy 0.3681 (Table~\ref{tab:joint_failure_main}). Clock-shift success remains zero. Figure~\ref{fig:factorization} shows how the components are intended to combine.

\begin{table}[t]
\centering\small
\caption{Three-seed joint-bank replication; means across seeds. These are separate metrics from shared-adapter literal-span matching.}
\label{tab:joint_failure_main}
\begin{tabular}{@{}lr@{}}\toprule
Metric & Score\\\midrule
Semantic top-1 & 0.7292\\
Temporal top-1 & 0.4167\\
Joint attribute/version top-1 & 0.2188\\
Answer exact match & 0.3681\\
Clock-shift success & 0.0000\\
Deleted-version/attribute recall & 0.0000\\\bottomrule
\end{tabular}
\end{table}

\paragraph{Interpreting the failure.}
FiLM and the router both receive features that can represent elapsed time. They still have to learn which validity interval contains the query time. Earlier attempts added training steps, replaced the temporal head with a deeper network, and added normalized raw time. None resolved the full-task failure (Appendix~\ref{app:joint_diagnostics}).

With shuffled clocks, combined routing accuracy is 0.3333, 0.2708, and 0.2292 across the three seeds. All exceed the mean live-clock score of 0.2188. A constant clock gives 0.3021. The supplied clock does not improve routing reliably in this replication. Direct distances to interval boundaries and separate temporal heads for each attribute remain possible next experiments.

Deleted items are no longer recalled, but accuracy on retained versions changes by as much as 0.0444 in absolute terms. Removing a unit therefore also affects some retained-version answers.

\begin{figure}[tbp]
\centering
\includegraphics[width=\linewidth]{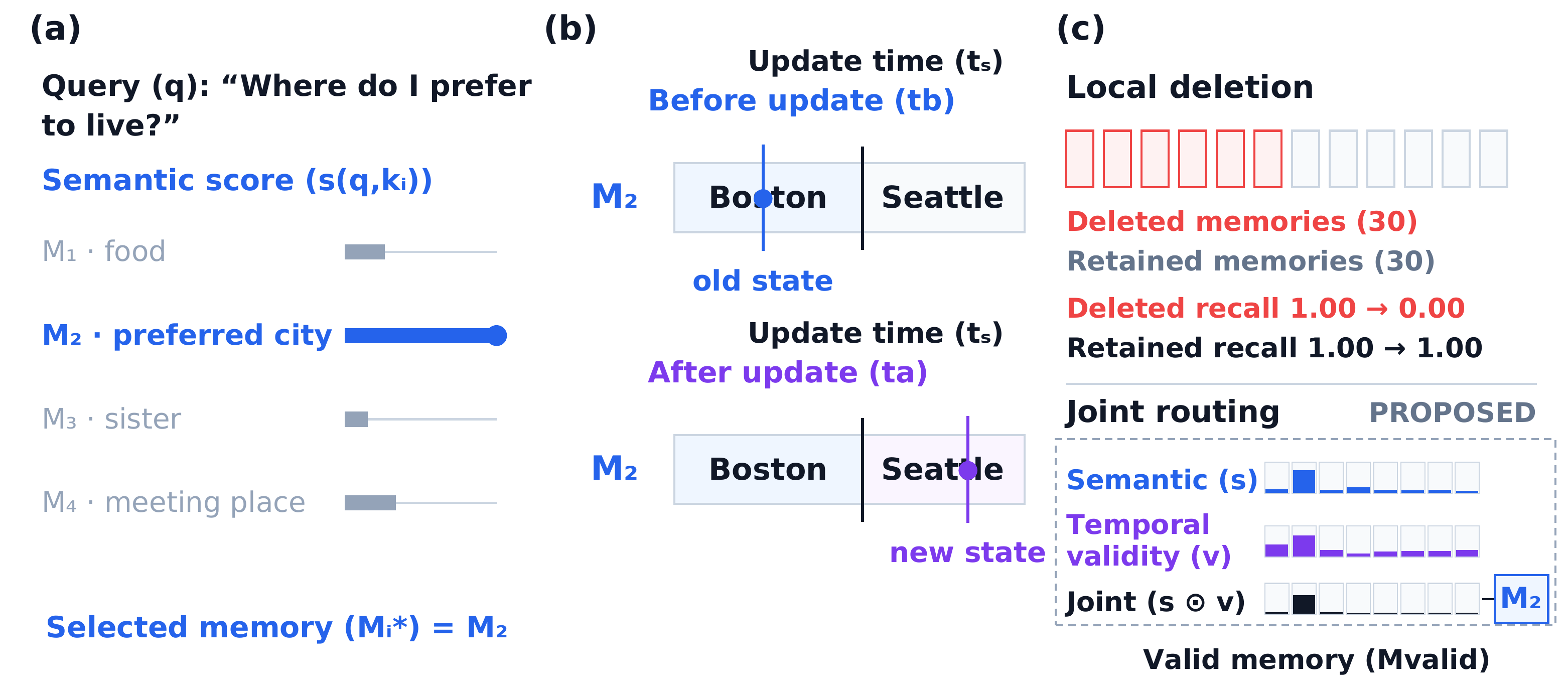}
\caption{\textbf{Attribute selection, version selection, and deletion.} (a) The question selects an attribute. (b) The clock should select its valid version. (c) An earlier bank with 30 forget and 30 retain memories illustrates deletion, separately from TOFU and the joint-bank replication. The dashed panel shows the intended combination. Table~\ref{tab:joint_failure_main} reports the joint system's measured performance.}
\label{fig:factorization}
\end{figure}

\subsection{Prompt size and computation costs}
\label{sec:efficiency}

The corrected test uses mean prompts of 1734.39 tokens for P0-asof, 143.94 for P1-asof, and 5.40 for each adapter. Their ranges are 258--3085, 134--152, and 5--6 tokens. The adapters use fewer prompt tokens, but these counts omit training and retrieval computation. Each adapter contains 40.91 MiB of trainable tensors. Table~\ref{tab:comparison} gives the earlier training-time-query prompt counts.

Earlier measurements give generation times of about 0.326 seconds for full-history prompting, 0.154 for retrieval, and 0.912 for the adapter. The runs use different decoding APIs and precomputed retrieval embeddings, so these times are not a matched speed comparison. Full adapter query times also include likelihood probes. Additive training takes a mean of 1689.7 seconds; a comparable FiLM training time is not established. Appendix~\ref{app:efficiency} reports separate earlier capacity and cost tests.

\section{Limitations}
\label{sec:limitations}

The corrected test uses synthetic facts, repeated question templates, three seeds, and one excluded query age of seven days. The model has already learned the tested facts. The test asks whether it can use them at a new elapsed age; it does not test unseen facts or generalization across a broad range of times. Earlier training-time and 55-pair results use different procedures, and separation from training is unverified for the earlier 55-pair test.

Answers are reviewed by a blinded assistant without independent human validation. The completion report describes checks on all evaluated training inputs, but the underlying source snapshots and logs were not supplied. We could verify report hashes and pooled counts, not repeat the full training audit. The uncertainty estimates describe these runs, and the clock and code comparisons are exploratory.

The ordinal codes are associated with write order during training. Keeping a code fixed across query times changes that association. The 24-token generation limit and the different input formats used by baselines and adapters may also affect their answers. When all pair scores are zero, clock controls cannot reveal whether the model uses time internally. The test does not identify a single cause for the failures.

Storage, deletion, and time-based selection are tested in separate configurations. Removing a memory unit prevents its direct contribution but does not prove that all related information is unrecoverable. Fewer prompt tokens also do not establish a lower total cost. That comparison must include training, updates, memory storage, and retrieval computation.

\section{Conclusion}
\label{sec:conclusion}

Engram Injection stores dated facts in learned parameters and supplies the current time through a separate clock channel. The component experiments show that facts can be stored and that removable units can support local deletion. Choosing the valid version remains unreliable. On 62 before/after pairs at a seven-day age excluded from training, both adapter methods score 0/62 under the primary review, compared with 59/62 for full-history prompting and 62/62 for retrieval. The joint bank also fails to select attributes and historical versions reliably. Learning a fact, selecting the relevant memory, and choosing when it applies are separate problems; the tested systems do not yet solve all three together.


\appendix
\section{Representing elapsed time with products}
\label{app:bilinear}

\begin{proof}[Proof of Proposition~\ref{prop:bilinear}]
Let
\[
a=\tau_q/T_k,
\qquad
b=t_i/T_k.
\]
Then
\[
\sin(a-b)
=
\sin a\cos b-\cos a\sin b,
\]
and
\[
\cos(a-b)
=
\cos a\cos b+\sin a\sin b.
\]
Both coordinates of $\chi_k(\tau_q-t_i)$ are therefore bilinear in the query-time and write-time phase encodings.
\end{proof}

The proof shows how products of two time encodings can represent their difference. Applying it to the model requires the hidden state to retain features of the fact's write time. The proposition does not establish that the shared adapter learns those features.

A single sinusoidal pair repeats, so different times can have the same phase. We combine several scales to reduce this ambiguity. Over the 180-day benchmark window, tested times separated by at least one day have a minimum full-encoding distance of 1.04. The largest single sinusoidal scale is unambiguous over only about 22 days. Longer-range ordering therefore requires more than one phase.

\section{Temporal-margin loss}
\label{app:margin}

\begin{proof}[Proof of Lemma~\ref{lem:margin_zero}]

The positive-part operator returns zero exactly when its input is nonpositive. The sum of its two nonnegative terms is therefore zero exactly when both inputs are nonpositive:
\begin{align*}
    \mu + m(\tau_{\rm before}) &\le 0,\\
    \mu - m(\tau_{\rm after}) &\le 0.
\end{align*}

Rearranging the two inequalities gives
\[
    m(\tau_{\rm before})\le -\mu < 0
    \qquad\text{and}\qquad
    m(\tau_{\rm after})\ge \mu > 0.
\]

These inequalities imply the signs required by the two preference metrics:
\[
    \mathrm{old\_before}
    =
    \Pr\bigl[m(\tau_{\rm before})<0\bigr],
    \qquad
    \mathrm{new\_after}
    =
    \Pr\bigl[m(\tau_{\rm after})>0\bigr].
\]

Zero margin loss therefore gives the required preference at both times for the evaluated triple \((q,\tau_{\rm before},\tau_{\rm after})\).
\end{proof}

The reverse implication does not hold. The metrics only require the correct sign of $m$, while zero loss also requires its magnitude to reach $\mu$.
\section{Additional experimental details}
\label{app:details}

\subsection{Datasets and temporal sampling}

PersonaFacts-100 contains 100 dated synthetic facts about identity, preferences, relations, and events over 180 days. The earlier main update test uses 25 old/new pairs within a 140-fact configuration. A separate test expands this to 100 update pairs.

ChronoRecall-125 has five question families with 25 questions each: absolute time (T-ABS), relative age (T-REL), event order (T-ORD), conflicting recent values (T-CONF), and age or staleness (T-AGE). We use PersonaFacts-1K for capacity tests. Parametric questions omit timestamp text. The external-memory baselines receive dated records and an explicit query day.

Shared-adapter training samples query ages on a log-spaced grid from one minute to 90 days. Earlier runs intended to exclude ages from 5 to 9 days, but a midpoint fallback could choose a time inside that band. The corrected sampler removes this fallback and applies the exclusion to every input evaluated for NLL or margin training (Section~\ref{sec:trained_code}). Only answer tokens contribute to the NLL loss. The inference prompt contains no fact text.

For an updated fact, we ask the same question before and after the change. The metrics $\mathrm{old\_before}$ and $\mathrm{new\_after}$ check whether the old and new answers, respectively, have higher likelihood at those times.

The earlier combined order-plus-margin experiment uses ten seeds on 25 pairs. Section~\ref{sec:components} reports it separately from the main three-seed comparison. At the 100-pair size, we use five seeds for the combined configuration and two for the clock-only baseline.

\subsection{Controls}

Table~\ref{tab:controls} lists the earlier controls. Tables~\ref{tab:comparison} and~\ref{tab:heldout} report P0, P1, and their as-of variants; Table~\ref{tab:trained_code} reports the corrected comparison.

\begin{table}[h]
\centering
\caption{
Temporal controls and ablations.
}
\label{tab:controls}
\small
\begin{tabular}{@{}p{0.18\linewidth}p{0.31\linewidth}p{0.43\linewidth}@{}}
\toprule
Arm &
Intervention &
Interpretation
\\
\midrule
P0 &
Facts as JSON in prompt &
Recurring long-context prefill cost
\\
P1 &
Retrieve dated chunks &
Text-based temporal reasoning baseline
\\
A1 &
Zero $\chi$ channel &
Removes runtime clock contribution
\\
Order only &
Add order tag, no temporal margin &
Tests whether order representation alone drives supersession
\\
Margin only &
Add temporal margin, no order tag &
Tests whether validity supervision drives supersession
\\
Shuffled time &
Permute timestamps &
Breaks fact--time alignment;
$\mathrm{new\_after}=0.28$
\\
Constant time &
Freeze runtime clock &
Removes query-time variation;
$\mathrm{new\_after}=0.44$
\\
\bottomrule
\end{tabular}
\end{table}

\subsection{ChronoRecall evaluation families}

Table~\ref{tab:families} describes what each ChronoRecall question family tests.

\begin{table}[h]
\centering
\caption{
ChronoRecall-125 evaluation families.
Evaluation questions omit timestamp text.
}
\label{tab:families}
\small
\begin{tabular}{@{}l p{0.34\linewidth} p{0.42\linewidth}@{}}
\toprule
Family &
Question type &
Intended diagnostic
\\
\midrule
T-ABS &
Absolute anchoring &
Calendar/elapsed categories from hidden clock state
\\
T-REL &
Relative age &
Query time relative to stored write time
\\
T-ORD &
Ordering &
Earlier/later relation between stored events
\\
T-CONF &
Recency conflict &
Old/new resolution across supersession
\\
T-AGE &
Age/staleness &
Coarse age and stale/current judgments
\\
\bottomrule
\end{tabular}
\end{table}

\subsection{External evaluation protocols}

In CounterFact, a dated update replaces a fact known by the base model. We measure success on the edited fact, preservation of held-out facts, preference for the original answer before the update, and edit success with the clock frozen. ROME and MEMIT are not rerun with the same backbone and implementation, so their published results are not matched comparisons.

The TOFU-derived test separates facts into forget and retain sets~\cite{maini2024tofu}. We measure recall before injection, after injection, and after removing the forget-set units. The operation deletes the addressed units exactly. The experiment tests whether forget-set answers return to their earlier behavior while retain-set answers are preserved; it does not establish that all traces of the forgotten information are removed.

\subsection{Reproducibility}

Each multi-seed run seeds Python, NumPy, and PyTorch, uses deterministic cuDNN, and saves adapter, optimizer, and random-number-generator states together. In a forced-interruption test, resuming produces zero per-step loss difference and the same final-weight hash.

\subsection{Corrected rerun and audit scope}
\label{app:clean_rerun}
The corrected experiment trains FiLM and additive adapters for each of seeds 0, 1, and 2, giving six fresh runs of 1200 optimizer steps. The fixed test contains 18, 20, and 24 eligible pairs, respectively. Every relevant fact is seven days old at its test query. Training draws with no allowed time are skipped and their reasons recorded. The completion report describes source inspection, the order of sampling and evaluation events, arithmetic checks on actual training inputs, comparisons with planned inputs, artifact hashes, and evaluation-row checks. It reports that all 127,760 training inputs obey the age exclusion and avoid exact test timestamps. The model learns the fact values during training but receives no test annotations or evaluation feedback.

We received the summary, conclusions, and completion report. The summary and conclusion hashes match the completion report, and pooled counts equal the sums of seed counts. The source snapshots, fixed test queries, training logs, and individual review decisions cited in that report were not supplied. We therefore report its exhaustive training-audit findings as reported checks, rather than claiming to have repeated them independently.

Retrieval uses all-MiniLM-L6-v2 with attention-mask mean pooling, normalized embeddings, and a 128-token embedding limit. We remove writes later than the query date, select the four most similar remaining writes by cosine similarity, and sort them by write time. All systems generate greedily with a 24-token limit and stop at EOS or a newline. The baselines use chat templates; the adapters use raw questions with numerical clock and code inputs.

Blinded assistant review labels 891 unique triples covering 1736 answer occurrences. It assigns 167 correct, 680 incorrect, and 44 ambiguous labels. Ambiguity appears in 71 occurrences. Strict scores reject ambiguous answers; upper-bound scores accept them. Literal matching and review disagree on 644 occurrences. Independent human validation is still absent. Paired clock tests and pair-bootstrap intervals within seeds are exploratory, apply only to the three tested seeds, and are not adjusted for multiple comparisons.

\section{Additional training and clock tests}
\label{app:optimization}

\subsection{Temporal-margin sensitivity}

We test two additional temporal-margin weights while keeping the other settings unchanged (Table~\ref{tab:lambda_sensitivity}). Both $\lambda=0.1$ and $\lambda=0.5$ give $\mathrm{new\_after}=0.84$, with high pre-update preference and recall. Each setting uses one seed, so these tests describe the runs rather than establish insensitivity to $\lambda$. The earlier ten-seed combined experiment uses a fixed selected coefficient.

\begin{table}[h]
\centering
\caption{
One-seed sensitivity checks for the temporal-margin coefficient.
}
\label{tab:lambda_sensitivity}
\small
\begin{tabular}{@{}lccc@{}}
\toprule
Setting &
$\mathrm{new\_after}$ &
$\mathrm{old\_before}$ &
Recall
\\
\midrule
$\lambda=0.1$ &
0.84 &
1.00 &
0.90
\\
$\lambda=0.5$ &
0.84 &
0.96 &
0.91
\\
\bottomrule
\end{tabular}
\end{table}

\subsection{Limits of periodic clock features}

Each sinusoidal pair repeats. A single scale $T_k$ distinguishes signed time differences only within the range
\[
|\Delta|<\pi T_k,
\]
while the $\log(1+\tau)$ feature continues to increase with time, but its sensitivity decreases at larger values.

Taking the dot product of two sinusoidal time encodings gives a similarity that depends only on their time difference:
\[
K(\Delta)
=
\sum_k
\cos(\Delta/T_k).
\]

This similarity peaks at 15 and has a half-height width of approximately 66 seconds. It also has secondary peaks near 1.8, 7.1, and 14.3 hours. Raw phase similarity therefore identifies closely aligned times but does not provide a straightforward way to choose a memory over days.

The attempted time-keyed memory bank learns a nonnegative weight $m_k$ for each scale:
\begin{equation}
K_M(\Delta)
=
\sum_k
m_k
\cos(\Delta/T_k),
\qquad
m_k\ge0.
\end{equation}
These experiments do not establish successful time-based routing with the learned weights.

\subsection{The clock gate during training}

During shared-adapter training, the loss falls and the mean clock gate $\alpha$ rises from zero to approximately 0.13. This shows that the clock path opens during optimization. Gate size alone does not show that the model chooses valid answers; that requires the answer and clock-control tests.

\subsection{Learning rate and LoRA rank}
\label{app:sweeps}

Table~\ref{tab:sweeps} gives earlier tests that vary one training setting at a time. In the tested range, post-update preference changes more with learning rate than with LoRA rank. The default learning rate is $3\times10^{-4}$ and the default rank is 8.

\begin{table}[h]
\centering
\caption{One-factor temporal sweeps.}
\label{tab:sweeps}
\small
\begin{tabular}{@{}llc@{}}
\toprule
Lever &
Value &
$\mathrm{new\_after}$
\\
\midrule
Learning rate &
$1\times10^{-4}$ &
0.68
\\
Learning rate &
$3\times10^{-4}$ &
0.79
\\
Learning rate &
$1\times10^{-3}$ &
\textbf{0.88}
\\
LoRA rank &
4 &
0.76
\\
LoRA rank &
8 &
0.79
\\
LoRA rank &
16 &
0.76
\\
\bottomrule
\end{tabular}
\end{table}

\section{Additional tests of joint routing}
\label{app:joint_diagnostics}

The joint-bank test uses a fixed set of attributes, with multiple historical versions for four of them. Queries include new paraphrases without timestamps, times inside excluded intervals, and times near update boundaries.

We first train the values using the correct routes supplied by an oracle. We then calibrate attribute selection and version selection separately before training them together. The final selection weight is
\[
g_{a,v}(q,\tau)
=
p_{\rm sem}(a\mid q)
p_{\rm time}(v\mid a,\tau).
\]

Small deterministic tests check the following properties:

\begin{itemize}
    \item changing only query semantics changes the selected semantic attribute;
    \item changing only $\tau$ changes the selected temporal version in the isolated construction;
    \item joint route mass sums to one;
    \item dense-mode deletion changes output by exactly the deleted unit's contribution up to numerical precision.
\end{itemize}

All four checks pass. The numerical residual for dense-mode deletion is approximately $5\times10^{-8}$. These checks verify the constructed components, separately from whether the trained router generalizes to the full test.

We try 800 temporal-head training steps, an MLP in place of the linear head, normalized raw-time features, and evaluation immediately after temporal calibration. Full-test version accuracy stays at $0.42$. The raw-time variant nearly solves a one-attribute test but fails across the complete set of attribute-specific intervals.

Shuffling time does not reduce joint performance relative to the live clock. The learned head therefore shows no reliable benefit from the supplied clock on the tested excluded intervals.

Possible next tests are to give relative time a more direct input path, use a separate temporal head for each attribute, or supply distances to the two validity boundaries:
\[
(\tau-t_0)
\qquad\text{and}\qquad
(t_1-\tau).
\]

\section{Additional cost, capacity, and safety tests}
\label{app:efficiency}

\paragraph{Recurring inference cost.}
Table~\ref{tab:efficiency_appendix} gives earlier measurements taken after memory was written, separately from Section~\ref{sec:efficiency}. They exclude training, update latency, and retrieval indexing. They also do not account for how many queries are needed to recover the cost of writing memory.

\begin{table}[h]
\centering
\caption{
Measured query-time cost at 140 stored facts.
}
\label{tab:efficiency_appendix}
\small
\begin{tabular}{@{}lcc@{}}
\toprule
Memory representation &
Prompt tokens/query &
Prefill time (s)
\\
\midrule
Facts in prompt &
2442 &
0.362
\\
Shared Parametric Memory &
\textbf{6} &
\textbf{0.042}
\\
\bottomrule
\end{tabular}
\end{table}

\paragraph{Memory footprint.}
The earlier rank-8 adapter occupies approximately 32 MB in memory and 36.5 MB when serialized. Its parameter count stays fixed as more facts are trained into it, but the number of facts it can recall does not necessarily stay high.

In an earlier context-length test, prompt-side inference uses 8.4 GB of peak memory at 8k tokens and 16.3 GB at 32k tokens. Engram Injection uses approximately 5.86 GB over the tested memory range. At 32k tokens, prompt-side inference uses $2.8\times$ as much peak VRAM.

\paragraph{Fixed-budget capacity sweep.}
With a fixed training budget, shared-adapter recall is 1.00 through $N=20$ facts, 0.88 at $N=60$, and 0.39 at $N=150$. The change in NLL on held-out general text ranges from $+0.06$ to $+0.18$ and does not increase consistently with $N$.

Memory recall declines before a consistent decline in general-text performance appears. The test does not distinguish limited training from interference or intrinsic capacity. The separate addressable bank keeps recall at 1.000 for $N=60$ and $N=150$, using approximately 220 KiB per fact.

\paragraph{Comparison with retrieval.}
Retrieval keeps the original dated text and supplies selected records at each query. Parametric memory learns the information during training. The main comparison tests both approaches on the same facts across three seeds, including baselines that filter records by query date. Its results are separate from these earlier capacity and cost experiments.

\paragraph{Exploratory safety diagnostics.}
In a small test of never-provided attributes, the shared adapter invents answers for 7 of 8 attributes. The frozen base invents answers for 0 of 4. Adding abstention examples to 20\% of training reduces adapter fabrication to 0 of 8 and raises recall on non-conflicting facts from 0.93 to 0.96.

In five prompt-injection tests, Engram Injection leaks in 3/5 cases and the base model in 2/5. These small tests are not repeated across seeds and cannot establish general safety performance.

\section*{Acknowledgments}

The authors used OpenAI GPT models to assist with manuscript writing and editing, and with writing code. The authors are responsible for the research ideas, mathematical claims, and final content of this paper.

\end{document}